\documentclass[preprint]{article}

\usepackage{neurips_2022}

\usepackage[utf8]{inputenc} 
\usepackage[T1]{fontenc}    
\usepackage{hyperref}       
\usepackage{url}            
\usepackage{booktabs}       
\usepackage{nicefrac}       
\usepackage{microtype}      
\usepackage{xcolor}         
\usepackage{graphicx}
\usepackage{subfigure}
\usepackage{natbib}         
\usepackage{notation}
\usepackage{times}
\usepackage{amsfonts,amssymb,amsmath,amsthm}
\usepackage[short]{optidef}
\usepackage{dsfont}
\usepackage{algorithm}
\usepackage{algorithmic}
\usepackage{thmtools}

\usepackage{xspace}

\newcommand{\LGapE}{\texttt{LinGapE}\xspace}

\newcommand{\SafeLGapE}{\texttt{Safe-LinGapE}\xspace}

\newcommand{\XYA}{\texttt{$\cX\cY$-Adaptive}\xspace}

\newcommand{\LG}{\texttt{LinGame}\xspace}

\newtheorem{theorem}{Theorem}
\newtheorem{lemma}[theorem]{Lemma} 
\newtheorem{proposition}[theorem]{Proposition} 
\newtheorem{remark}[theorem]{Remark}

\newtheorem{assumption}[theorem]{Assumption}

\makeatletter
\newcommand{\myitem}[1]{%
\item[#1]\protected@edef\@currentlabel{#1}%
}
\makeatother

\definecolor{aqua}{rgb}{0.0, 1.0, 1.0}

\definecolor{linkcolor}{RGB}{83,83,182}
\hypersetup{
    colorlinks=true,
    citecolor=linkcolor,
    linkcolor=linkcolor
}

\title{Price of Safety in Linear Best Arm Identification}

\author{%
  Xuedong Shang \\
  Sequel Team\\
  Inria Lille\\
  \texttt{xuedong.shang1@gmail.com}
  \And
  Igor Colin \\
  Noah's Ark Paris\\
  Huawei\\
  \texttt{igor.colin@huawei.com} \\
  \And
  Merwan Barlier \\
  Noah's Ark Paris\\
  Huawei\\
  \texttt{merwan.barlier@huawei.com}
  \And
  Hamza Cherkaoui \\
  Noah's Ark Paris\\
  Huawei\\
  \texttt{hamza.cherkaoui@huawei.com} \\
}

\begin{document}

\maketitle

\begin{abstract}
We introduce the safe best-arm identification framework with linear feedback, where the agent is subject to some stage-wise safety constraint that linearly depends on an unknown parameter vector. The agent must take actions in a conservative way so as to ensure that the safety constraint is not violated with high probability at each round. Ways of leveraging the linear structure for ensuring safety has been studied for regret minimization, but not for best-arm identification to the best our knowledge. We propose a gap-based algorithm that achieves meaningful sample complexity while ensuring the stage-wise safety. We show that we pay an extra term in the sample complexity due to the forced exploration phase incurred by the additional safety constraint. Experimental illustrations are provided to justify the design of our algorithm.
\end{abstract}

\section{Introduction}\label{sec:intro}

Stochastic multi-armed bandit (MAB) models the problem of sequentially allocating resources to a defined set of competing actions (arms) based on successive \emph{partially observable} feedback. In its simplest form, feedback (rewards) of playing an arm are generated according to an \emph{unknown} underlying probability distribution. The goal in this paper is to identify the optimal arm: the agent plays an arm at each round, and makes a guess (decision) for the arm with the largest mean reward when she stops according to some criterion. Such learning goal is called \emph{best-arm identification} (BAI). In BAI, the agent rather cares about the quality of her final decision than minimizing the potential losses incurred during the whole learning phase as for another common learning goal, namely \emph{regret minimization} (see~\citealt{lattimore2018} for a survey on MAB). 
The problem setting we consider in this work is thus a particular instance of \emph{pure exploration}~\citep{bubeck2009pure}.

BAI for linear bandits extends the vanilla BAI to a setting where the arm space $\cX$ consists of feature vectors $\bx\in\R^d$ and the expected reward of playing arm $\bx$ is the inner product $\bx^\top\btheta^\star$. We call $\btheta^\star$ the regression parameter and at this stage, it is unknown to the agent. Such framework is first investigated by~\citet{soare2014linear} in a setting where we try to confidently tell the optimal arm with a small number of arm plays.

BAI for linear bandits finds its applications in many real scenarios. Typically, we can think of ad display optimization: a website seeks to identify the best-performing ad display design. 
Note that it is arguable whether we need regret minimization or best-arm identification for those situations: a reasonable guess is that it is often subject to the real business needs and constraints.

However for some scenarios including physical systems, linear BAI algorithms are not straightforwardly applicable, which is often due to extra safety requirements. For example, we consider a telecommunication company that wants to optimize the power allocation over a set of base stations. A trial phase is allowed to identify the optimal (or near-optimal) configuration. During this trial phase the overall cost is not essential as long as we can find the best configuration. However, a minimum quality of service is required. This situation, where a hard operational safety constraint is expected to be complied with while exploring, can be naturally modeled as a safe linear bandits BAI problem. 

In such a \emph{safe} setting, a safety constraint is added on top of the classical linear BAI setting. The safety constraint thus writes $\gamma\bx^\top\btheta^\star \geq \eta_0$ where $\gamma\in[0,1]$ is the level the arm is pulled (that we call \emph{safety risk parameter}) and $\eta_0$ a safety threshold. This particular arm structure is typically encountered in real-world physics-based applications. For example, in many signal engineering field, we often have the ability to control the level of input pulse in response to safety requirements (this could be the case for instance with voltage input as can be found in applications such as power grid, antenna tower, etc). Another example is in drug discovery, where we would also like to adjust the dose level of the drug that we are testing. Within this context, we want to ensure the \emph{stage-wise} safety by only playing arms with safe risk parameters (at least with a high probability $1 - \delta_\mathrm{S}$). 

A BAI strategy/algorithm shall contain three components including (i) a \emph{sampling rule} that guides the agent which arm to sample at each round, (ii) a \emph{stopping rule} that tells the agent when to stop and (iii) a \emph{decision rule} that outputs the final guess of the best arm. In this work, the focus is put on the sampling rule and how to make it adapt to the safety constraints. 

The design of such a sampling rule can take inspiration of the three existing types of sampling rule for linear bandits BAI: \emph{elimination-based}, \emph{game-based} and \emph{gap-based}. Elimination-based sampling rules, like \XYA~\citep{soare2014linear}, operate in phases and successively eliminate sub-optimal directions. Such sampling rules often rely on solving a complex optimization problem at each round that we would like to avoid. Game-based sampling rules, like \LG~\citep{degenne2020game}, model the problem from a game-theoretical point of view and are built upon the lower bound, which are not easily applicable in our case since the lower bound with safety constraints is not clear. Therefore, we opt for gap-based methods. Gap-based sampling rules, like \LGapE~\citep{xu2018linear}, choose to play the arm that reduces most the uncertainty of the gaps between the empirical optimal arm and the rest.

\paragraph{Contributions.}
The main contributions of the paper are the following: (i) we introduce a framework including the formulation of a novel goal in order to study linear bandits BAI problems under stage-wise safety constraints and propose a gap-based algorithm to address these problems; (ii) we provide a sample complexity analysis; (iii) we finally provide experimental illustrations.

\paragraph{Outline.}
The rest of the paper is organized as follows: After a review of related work in Section~\ref{sec:related}, we formalize the problem setting in Section~\ref{sec:formulation}. The main algorithm is introduced in Section~\ref{sec:safebai}. We further provide some experimental illustrations in Section~\ref{sec:experiments} before concluding.
\section{Related Work}\label{sec:related}

\paragraph{Best-arm identification.} Two major frameworks exist for BAI: (a) fixed-budget; (b) fixed-confidence. Fixed-budget setting aims to minimize the probability of misidentifying the best arm within a given budget~\citep{bubeck2009pure,audibert2010budget,gabillon2012ugape,karnin2013sha,carpentier2016budget}. We investigate the fixed-confidence setting in this paper, introduced by~\citet{even-dar2003confidence}, for whom the objective is to ensure that the algorithm returns the best arm with high confidence, while minimizing the total number of samples to be used. Fixed-confidence BAI has been extensively studied in its classical form\footnote{Stochastic bandits without side information.}~\citep{even-dar2003confidence,kalyanakrishnan2012lucb,gabillon2012ugape,jamieson2014lilucb,garivier2016tracknstop,qin2017ttei,degenne2019game,menard2019lma,shang2020t3c} as well as under linear payoffs or beyond~\citep{soare2014linear,tao2018alba,xu2018linear,zaki2019maxoverlap,fiez2019transductive,kazerouni2019glb,degenne2020game,katz-samuels2020practical,zaki2020linear,jedra2020linear}. None of the above is subject to any kind of extra constraints.

\paragraph{Multi-armed bandits with constraints.} There are two different lines of research where constraints are added on top of MAB problems. The first line considers constrained resources consumption over different arms~\citep{badanidiyuru2013knapsacks,badanidiyuru2014resource,wu2015constrained,agrawal2016knapsacks}, which is out of the scope of this work. We are interested in the second line, where a safety constraint is required on the chosen actions at each stage of the algorithm. 

Most of previous work on MAB with \emph{stage-wise} safety constraints only care about regret minimization~\citep{wu2016conservative,kazerouni2017conservative,moradipari2019safe,amani2019safe,khezeli2020safe,amani2020decentralized,amani2020generalized}. For example, \citet{kazerouni2017conservative} provide regret bounds based on the gap between the optimal (but probably unsafe) policy and the safe policy.

To the best of our knowledge, the only study about BAI with safety constraints is developed by~\citet{wang2021safe}. They consider a \emph{linear response} setting that is more closely related to BAI for stochastic bandits with an assumption on partial pulling of the arms than to linear bandits. Our setting, on the other hand, considers the classic linear bandits as studied for example by~\citet{soare2014linear}.

\paragraph{Safe learning (exploration).}
In this work, our focus is put on ensuring exploration safety during a BAI learning process. Although not straightforwardly applicable to our problem, it is nevertheless worth mentioning that the very topic and relevant techniques have been increasingly investigated in a more general context. A related example is safe exploration (in unknown and stochastic environments) in reinforcement learning (see \eg \citealt{dalal2018safe,ding2021safe,xu2021crpo}), control theory (see \eg \citealt{brunke2021safe}), as well as optimization problems (see \eg \citealt{sui2015safe,sui2018safe,usmanova2019safe}), in which policy updates also depend on a trade-off between objective improvement and constraint satisfaction similar to what is required in this paper. Some works have already emerged as critical components to many recent top-notch applications such as autonomous driving (see \eg \citealt{leurent2020thesis}), which further consolidate the importance of relevant research.

\section{Problem Formulation}\label{sec:formulation}

We start by introducing the formal learning framework. For any integer $n$, we denote by $[n]$ the set of integers $\{1,\ldots,n\}$. We denote by $\|\cdot\|$ the Euclidean norm. We also denote, given a positive semi-definite matrix $\bA\in\bbR^{d\times d}$, by $\forall \bx\in\bbR^d, \quad \|\bx\|_{\bA} = \sqrt{\bx^\top\bA\bx}$ the Mahalanobis norm. We further denote by $\Delta_K \triangleq \{\bp : \sum_{i=1}^K p_i = 1\}$ the simplex of dimension $K$.


\subsection{Best-arm identification}

We first recall the formulation of BAI for stochastic MAB. Consider a set of arms/contexts $\cX = \{\bx_1, \ldots, \bx_K\}$ where $\forall i\in[K], \bx_i \in \bbR^d$. By abuse of notation, we use $\{1,\ldots,K\}$ to represent the $K$ arms when there is no ambiguity in the rest of the paper. 

As aforementioned, each BAI strategy includes first a sampling rule, which selects an arm $\xhat_t\in\cX$ at round $t$. A vector of rewards $\br_t = (r_{t,1},\ldots,r_{t,K})$ is then generated for all arms independently from past observations, but only $r_{t,\hat{\bx}_t}$ is revealed to the agent that we denote as $r_t$ by next. The agent receives a noisy observation of the linear combination of $\xhat_t$ and $\btheta^\star$ as payoff,
\[
    r_t = \xhat_t^\top\bthetas + \epsilon_t\,.
\]
Let $\cF_t$ be the $\sigma$-algebra generated by $(\xhat_1,r_1,\ldots,\xhat_t,r_t)$, then $\xhat_t$ is $\cF_{t-1}$-measurable. We then make the following common assumption on the noise $\epsilon_t$.

\begin{assumption}[Sub-Gaussianality of the noise]\label{asmpt:noise}
We assume that the noise is conditioned $R$-sub-Gaussian for some constant $R>0$. That is, $\forall \lambda\in\R$,
\[
    \EE{\exp(\lambda\epsilon_t)|\cF_{t-1}} \leq \exp(\lambda^2 R^2/2)\,.
\]
\end{assumption}

Second, the agent needs a stopping rule $\tau$ to decide when to stop the learning. $\tau$ can be modeled as a stopping time with respect to the filtration $\left(\cF_{t}\right)_{t \in \mathbb{N}}$.

Finally, the agent needs a $\cF_{\tau}$-measurable decision rule $a_\tau$, which returns a guess for the best arm when they stop.

\subsection{Estimation of the regression parameter}\label{sec:formulation.regression}

In order to derive the optimal arm, the agent has to estimate the regression parameter $\bthetas$ as precisely as possible. For $t > 0$, let $\bX_t = (\xhat_1,\ldots,\xhat_t)$ be a sequence of sampled arms, and $(r_1,\ldots,r_t)$ be the corresponding observations. To estimate $\bthetas$ based on the adaptive sequence of observations, one may use the \emph{regularized least-squares} estimator
\begin{align}\label{eq:update_mean}
    \hat{\btheta}_{\bX_t}^{\lambda} = (\lambda \bI_d + \bA_{\bX_t})^{-1}\bb_{\bX_t}\enspace,
\end{align}
where $\bA_{\bX_t}$ and $\bb_{\bX_t}$ are the design matrix and the response vector respectively given by $\bA_{\bX_t} \triangleq \sum_{s=1}^t \xhat_s\xhat_s\top$ and $\bb_{\bX_t} \triangleq \sum_{s=1}^t \xhat_s r_s$.
When clear from context, we can simply write $\hat{\btheta}_t^{\lambda}$ (resp. $\bA_t$ and $\bb_t$) instead of $\hat{\btheta}_{\bX_t}^{\lambda}$ (resp. $\bA_{\bX_t}$ and $\bb_{\bX_t}$). By next, we also denote by $\bA_t^\lambda \triangleq \lambda \1_d + \bA_t$ the regularized design matrix. 

The following boundedness assumption then allows us to derive a confidence ellipsoid of the regression parameter given by~\citet{abbasi-yadkori2011linear}. The result is restated in Theorem~\ref{th:confidence_set}.

\begin{assumption}[Boundedness]\label{asmpt:bound}
We assume that $\forall \bx\in\cX$, $\|\bx\|\leq L$ and that the true regression parameter is also bounded: $\| \bthetas \| \leq D$.
\end{assumption}

\begin{theorem}[\citealt{abbasi-yadkori2011linear}, Theorem 2]
\label{th:confidence_set}
Given the assumptions above, if for all $t \geq 1$, $||\bx_t||_2 \leq L$, for any $\delta > 0$, then with probability at least $1-\delta$, for all $t\geq 0$, $\bthetas$ lies in the set
\begin{align*}
    \mathcal{C}_t^{(\mathrm{r})} = \bigg \{ \btheta \in \bbR^d \;:\; ||\hat{\btheta}_t^\lambda - \btheta||_{\bA_{\bX_t}}
    \leq \beta_t(\delta)\bigg \} ~ \text{where} ~ \beta_t(\delta) \triangleq R \sqrt{d \log \bigg ( \frac{1+tL^2/\lambda}{\delta}\bigg )} + \lambda^{1/2}D \enspace.
\end{align*}

\end{theorem}

In order to find the best arm in the minimum time, the agent has to select arms allowing to shrink that set as fast as possible.

\subsection{Adding safety constraints}
\label{sec:safe_bandits}

The objective of this paper is to study the linear bandits BAI problem while ensuring stage-wise safety with high probability. We shall make some slight modifications to the previous problem setting, without which nothing meaningful can be done. Indeed, it is necessary to have a minimum knowledge about the \emph{safety risk} of each arm at the very beginning of the learning procedure, that allows us to obtain a first estimation of $\btheta^\star$ (often corresponds to a forced-exploration phase). We integrate this notion of safety risk into the problem setting described below.

We again consider a set of $K$ arms/contexts $\cX = \{\bx_1, \ldots, \bx_K\}$ that spans $\bbR^d$, where for all $k \in [K]$, $\bx_k \in \bbR^d$. In addition, we introduce a \emph{safety risk parameter} $\gamma$. Given a safety risk parameter $\gamma$, the agent is allowed to only \emph{partially pull} an arm $\bx_k$, which corresponds to a partial reward $\gamma\bx_k^\top\btheta^\star$.

This safety risk parameter is similar to what is proposed by~\citet{amani2019safe} and can be interpreted as the border of a safe region. Initially, we only have a minimum knowledge of the border, and by continuously collecting new information about the arms mean reward, we will also gather information about their safety risk, thus increasingly expand the border in the directions that we have confidence to be safe at the same time.
One reason that this assumption is not very common in traditional linear bandit literature is because, without safety considerations, the linear structure would impose the parameter $\gamma$ to be set to its maximum value in order to get the optimal reward. 

The new learning process thus proceeds as follow: At each round $t$, the agent chooses an action $\hat{\bx}_t$ and observes the corresponding reward 
\[r_t = \hat{\bx}_t^\top \btheta^\star + \epsilon_t\,,\] where there exists $\gamma_t \in [0,1]$ and $k\in[K]$ such that $\hat{\bx}_t=\gamma_t\bx_k$. Note that Assumption~\ref{asmpt:noise} on the noise $\epsilon_t$ still applies in the new setting.



\paragraph{Safety constraint.}
We now introduce formally the definition of safety that is to be employed in this work. We define the following \emph{safe set}: 
\[
    \safeset(\bmu^\star) \triangleq \left\{ \gamma\bx; \bx\in\cX, \gamma\in[0,1], \gamma\bx^\top \bmu^\star \geq \eta_0 \right\}\enspace,
\]
where $\bmu^\star \in \bbR^d$ is the linear parameter of the safety signal and $\eta_0 < 0$ is the safety threshold set by the agent. In addition, we consider that the agent observes simultaneously the reward \emph{and} the safety signal after each action it takes. We assume that the safety signal comes with a Gaussian noise, similar to the reward signal's but independent.
The agent is only allowed to choose actions $\hat{\bx}$ which are safe with probability $1 - \delta_\mathrm{s}$ for some fixed confidence level $0 < \delta_\mathrm{s} < 1$, that is 
\[
 \bbP\left( \bigcup_{t = 1}^\tau \hat{\bx}_t \in \safeset(\bmu^\star) \right) \geq 1 - \delta_\mathrm{s} \enspace,
\]
where $\hat{\bx}_1, \ldots, \hat{\bx}_\tau$ are the arms pulled during the BAI procedure and $\tau > 0$ is the time at which the stopping condition is reached. It is assumed that for each arm $\bx \in \cX$, there exists a \emph{safety risk threshold} $0 < \gamma_\bx \leq 1$ such that $\gamma_\bx \bx$ is safe\footnote{Notably, $\forall \gamma\leq\gamma_\bx$, $\gamma\bx$ is safe, and $\forall \gamma>\gamma_\bx$, $\gamma\bx$ is unsafe.}. We further assume that the agent knows some $\bar{\gamma}$ such that $\bar{\gamma} \cX \subseteq \mathcal{S}(\bmu^\star)$\footnote{Note that here it is not necessary to have one $\bar{\gamma}_k$ for each individual arm $k$ as we can simply take $\bar{\gamma}\triangleq\min_{k\in[K]}{\bar{\gamma}_k}$.}. The knowledge of such a $\bar{\gamma}$ allows us to launch the initial forced-exploration phase that is detailed in Section~\ref{sec:forced_exploration}.


\paragraph{Learning objective.} Let $\hat{\btheta}$ (resp.\ $\hat{\bmu}$) be an estimate of the reward parameter $\btheta^\star$ (resp.\ the safety parameter $\bmu^\star$) and $\Pi_\safeset(\hat{\btheta})\triangleq\argmax_{\bx\in\safeset}\bx^\top\hat{\btheta}$ be the empirical best arm among a given set $\safeset$ of arms. We can thus denote by $\bx^\star \triangleq \Pi_{\safeset(\bmu^\star)}(\btheta^\star)$ the best safe arm, and make the following technical assumption.

\begin{assumption}[Uniqueness of the best arm, optional]\label{asmpt:unique}
We assume that the (unknown) best safe arm is unique and we denote it by $\bx^\star=\argmax_{\bx\in \mathcal{S}(\bmu^\star)}\bx^\top\btheta^\star=\Pi_{\safeset(\bmu^\star)}(\btheta^\star)$.
\end{assumption}

We study the fixed-confidence BAI setting in this paper. More precisely, we focus on the safe $(\varepsilon, \delta_\mathrm{r})$-best-arm identification problem, where the objective is to return a safe arm $\bx_\mathrm{max}(\tau)$ such that
\begin{equation}
    \label{crit:bai}
    \bbP((\bx_\mathrm{max}(\tau)-\bx^\star)^\top\btheta^\star\leq  \varepsilon)\geq 1- \delta_\mathrm{r}
\end{equation}
with a minimum of samples and by choosing only actions satisfying the safety requirements.

\section{Safe Best-Arm Identification}\label{sec:safebai}

We propose a method operating in two distinct phases. First, a forced-exploration phase during which pessimistic estimates of the safety parameters are learned. Then, a safe best arm identification phase, in the form of a constrained version of \LGapE. 

\subsection{Forced-exploration phase}
\label{sec:forced_exploration}

An initial phase of forced exploration allows us to refine the estimation of $\bthetas$ and $\bmus$. During this phase, the agent uniformly samples arms from a known safe set (typically $\overline{\gamma} \cX$ as defined in Section \ref{sec:safe_bandits}). This phase lasts until the optimal safe arm lies in the estimated conservative safe set. After $T_\mathrm{FE}$ rounds of forced exploration, a first estimate $\hat{\btheta}_\mathrm{FE}$ (resp.\ $\hat{\bmu}_\mathrm{FE}$) of $\bthetas$ (resp.\ $\bmus$) is obtained, along with a confidence set $\mathcal{C}^{(\mathrm{r})}_\mathrm{FE}(\delta'_\mathrm{s})$ (resp.\ $\mathcal{C}^{(\mathrm{s})}_\mathrm{FE}(\delta'_\mathrm{s})$), allowing to derive for each arm $\bx_k$ a conservative estimate of its maximum safety threshold $\overline{\gamma}_k$, that is
\[
    \overline{\gamma}_k \triangleq
    \begin{cases}
    \dfrac{\eta_0}{\min_{\bmu \in \mathcal{C}^{(\mathrm{s})}_\mathrm{FE}} \bx_k^\top \bmu} & \text{if } \min_{\bmu \in \mathcal{C}^{(\mathrm{s})}_\mathrm{FE}} \bx_k^\top \bmu < \eta_0 \enspace,\\
    1 & \text{otherwise} \enspace.
    \end{cases}
\]
Equivalently, $\overline{\gamma}_k$ may be formulated as follows:
\[
    \overline{\gamma}_k = 1 - \left[ 1-\dfrac{\eta_0}{\bx_k^\top \hat{\bmu}_\mathrm{FE} - \beta_{T_\mathrm{FE}}(\delta'_\mathrm{s}) \|\bx_k\|_{\bA^{-1}_\mathrm{FE}}} \right]_+ \enspace,
\]
where 0 < $\delta'_\mathrm{S}$ < 1 is a fixed confidence level which will determine how likely an arm pull will break the safety constraints at each step of the BAI procedure. Note that $\forall a\in\R, [a]_+ \triangleq \max\{a,0\}$. Therefore, if $\bx_k$ is unsafe, choosing such $\overline{\gamma}_k$ ensures that pulling $\overline{\gamma}_k\bx_k$ is safe with high probability.

\begin{remark}
    Notice that once the pessimist estimates are computed, we only focus on the arm set $\{ \overline{\gamma}_k \bx_k \text{, } 1 \leq k \leq K \}$ rather than $\{ \gamma_k \bx_k \text{, } 1 \leq \gamma_k \leq \overline{\gamma}_k \text{, } 1 \leq k \leq K \}$. This is simply due to the fact that one wants to choose the highest possible amplitude in order to get the maximum amount of information, therefore lower $\gamma$'s would never be selected during the BAI procedure. We provide a formal proof of this claim in the supplementary materials.
\end{remark}

With this refined estimation of a safe amplitude for each arm, one can move to the second phase.

\subsection{Safe best-arm identification}

\begin{algorithm}[t]
  \caption{\SafeLGapE (conservative)}
  \label{alg:safe-bai}
  \begin{algorithmic}[1]
    \STATE \textbf{Input}: accuracy $\varepsilon$, safety confidence level $\delta'_\mathrm{s}$, BAI confidence level $\delta_\mathrm{r}$, noise level $R$, norm $D$, regularization parameter $\lambda$, safety bound $\overline{\gamma}$
    \STATE \textbf{Initialization}: $\bA_0 \leftarrow \lambda \bI_d, \bb^{(\mathrm{r})}_0 \leftarrow \mathbf{0}, \bb^{(\mathrm{s})}_0 \leftarrow \mathbf{0}, \safeset' \leftarrow \overline{\gamma} \cX$
    \FOR{$t=1,\dots,T_\mathrm{FE}$}
    \STATE Uniformly sample arms $\hat\bx_t \in \mathcal{S'}$ 
    \STATE Observe reward $r_t = \hat{\bx}_t^\top\bthetas + \epsilon_t$
    \STATE Observe safety $s_t = \hat{\bx}_t^\top\bmus + \epsilon'_t$
    \STATE Update $\bA_t$, $\bb^{(\mathrm{r})}_t$ and $\bb^{(\mathrm{s})}_t$ following~\eqref{eq:update_mean}
    \ENDFOR

    \FOR{$t=T_\mathrm{FE}+1,\dots$}
    \STATE Compute the conservative/pessimistic safe set $\mathcal{S}_t$
    \STATE $\bx_\mathrm{max}(t), \bx_\mathrm{opt}(t), B(t) \gets$ select-direction$(t)$
    \IF{$B(t)\leq\varepsilon$}
    \STATE \textbf{return} $\bx_\mathrm{max}(t)$ as the decision
    \ENDIF
    \label{safe_lingape:best-gap}
    \STATE Pull the arm $\bx_\mathrm{pul}(t)$ according to \eqref{crit:greedy_safe} or \eqref{crit:prob_safe}
    \STATE Observe reward $r_t = \bx_\mathrm{pul}(t)^\top\bthetas + \epsilon_t$
    \STATE Update $\bA_t$ and $\bb^{(\mathrm{r})}_t$ following~\eqref{eq:update_mean}
    \ENDFOR
  \end{algorithmic}
  \label{alg:safe_lingape}
\end{algorithm}
Once the forced-exploration phase is over, we apply a gap-based sampling rule on the new bandit model with the filtered actions. It is worth noting that contrary to a classical BAI algorithm, sampling (potentially partially) an arm brings information not only on the reward, but also on the safe set which contains a growing number of safe actions. Although not included in our theoretical analysis, dynamic updates of the safety coefficients may decrease the overall running time, as evidenced in our experiments in Section~\ref{sec:experiments}.

Initially, the algorithm builds a conservative set of safe actions $\mathcal{S}_t$ by setting for each arm $k\in[K]$ the safety parameter $\overline{\gamma}_k$, which can be interpreted as the \textit{maximal proportion of the arm allowed to be pulled}. Note that the safety set $\mathcal{S}_\mathrm{FE}$ can be rewritten as:
\[
  \mathcal{S}_\mathrm{FE} = \bigcap_{\bmu \in \mathcal{C}^{(\mathrm{s})}_\mathrm{FE}} \mathcal{S}(\bmu)\enspace .
\]
Then, for all $t\geq T_\mathrm{FE}+1$, the agent repeats the following steps. First, it estimates the regression parameter $\hat{\btheta}_t^{\lambda}$ using~\eqref{eq:update_mean}, and uses Theorem~\ref{th:confidence_set} to compute the confidence ellipsoid $\mathcal{C}^{(\mathrm{r})}_t$. 

Since all the arms considered are safe with high probability, the agent can then apply the arm selection strategy proposed by \LGapE. For each arm $\bx_k$, upper bounds $\overline{r}_{k,t}=\max_{\btheta \in \mathcal{C}_t}\bx_k^\top \btheta$ are computed. The empirical best arm and the optimistic best arm in $\mathcal{S}_t$ can thus be identified and the agent can pull the safe arm maximizing some criterion~\eqref{crit:greedy_safe} or~\eqref{crit:prob_safe} in order to disambiguate between those two arms.

Those criteria are the following. First, a greedy one, where the selected arm $\bx_\mathrm{pul}(t)$ is solution of
\begin{equation}
  \label{crit:greedy_safe}
  \min_{\mathbf{x} \in \mathcal{S}_t} \; \by(t)^\top \left(\mathbf{A}_{t-1} + \mathbf{x}\mathbf{x}^\top\right)^{-1} \by(t) \enspace,
  \tag{G}
\end{equation}
where $\mathbf{A}_t$ is defined as in Section \ref{sec:formulation.regression}, $\by(t) \triangleq \bx_\mathrm{max}(t) - \bx_\mathrm{opt}(t)$ and $\bx_\mathrm{max}$ and $\bx_\mathrm{opt}$ are the arms identified using Algorithm~\ref{alg:safe-bai-conservative} (or an optimistic variant, as proposed in Section~\ref{sec:experiments}). The second criterion relies on a rounding procedure:
\begin{equation}
  \label{crit:prob_safe}
  \bx_\mathrm{pul}(t) \in \argmin_{\bx_k \in \mathcal{S}_t} \; \frac{N_t(\bx_k)}{\nu^\star_k(t)}\,,
  \tag{R}
\end{equation}
where we use $N_t(\bx_k)$ (or simply $N_t(k)$ when clear from context) to represent the weighted number of pulls of arm $\bx_k$ before time $t$ and $\bnu^\star(t)$ if the solution of the following optimization problem
\[
    \min_{\bnu \in \Delta_K} \by(t)^\top \left( \mathbf{A}_{t-1}(\lambda) + \sum_{k = 1}^K \nu_k \overline{\gamma}_k^2 \mathbf{x}_k \mathbf{x}_k^\top \right)^{-1} \by(t) \enspace.
\]
As for the unsafe version of \LGapE, the two criteria provide comparable performances, although the rounding procedure allows for a more refined theoretical analysis. For the sake of completeness, we state several key results of~\citet{xu2018linear} in the Appendix.

\paragraph{Stopping rule and decision rule.}
The stopping rule is adapted from the Chernoff stopping rule, with the addition of the pessimist safety estimate. For $t > 0$, let
\[
    B(t) \triangleq \by(t)^\top \hat{\btheta}_t + \| \by(t) \|_{\bA_t^{-1}} C_t \enspace.
\]
Then, the stopping time $\tau$ is the first iteration where $B(\tau) \leq \varepsilon$. And finally, the decision rule is simply recommending the empirical best arm. Combining the sampling rule, stopping rule and decision rule, we obtain the algorithm as displayed in Algorithm~\ref{alg:safe_lingape}.

\begin{algorithm}[t]
  \caption{\SafeLGapE select-direction (conservative)}
  \label{alg:safe-bai-conservative}
  \begin{algorithmic}[1]
    \STATE \textbf{Input}: time $t$
    \STATE Update the estimate of $\hat{\btheta}_t^\lambda \leftarrow \bA_t^{-1}\bb_t$
    \STATE $\mathbf{x}_\mathrm{max}(t) \gets \argmax_{\mathbf{x} \in \mathcal{S}_t} \mathbf{x}^\top \hat{\btheta}_t^\lambda$.
    \STATE $\mathbf{x}_\mathrm{opt}(t) \gets \argmax_{\mathbf{x} \in \mathcal{S}_t} (\mathbf{x}_\mathrm{max}(t) - \mathbf{x})^\top \hat{\btheta}_t^\lambda$
    \STATE $B(t) \gets \max_{\mathbf{x} \in \mathcal{S}_t} (\mathbf{x}_\mathrm{max}(t) - \mathbf{x})^\top \hat{\btheta}_t^\lambda$
    \STATE \textbf{return} $\bx_\mathrm{max}(t), \bx_\mathrm{opt}(t), B(t)$
  \end{algorithmic}
\end{algorithm}

\subsection{Analysis}

We propose a theoretical analysis of the safe BAI algorithm \SafeLGapE, ensuring that for a small enough $\delta'_\mathrm{S}$, the output arm is both safe and nearly optimal. The following result details the impact of the forced exploration phase on the BAI procedure.

\begin{restatable}{proposition}{restatemain}
    \label{thm:bound}
    Let $0 < \delta_\mathrm{s} < 1$ be the required confidence level of safety and $0 < \delta_\mathrm{r} < 1$ be the confidence level of best-arm identification.
    If the confidence $\delta'_\mathrm{s}$ of the safety coefficients lower confidence bound is such that
    \[
        \delta'_\mathrm{s} \leq \frac{\delta_\mathrm{s} - \delta_\mathrm{r}}{\big(1 - \delta_\mathrm{r}\big)\big(|\mathcal{X}| + C(\delta_\mathrm{r})\big)} \enspace,
    \]
    then following the procedure described in Algorithm~\ref{alg:safe_lingape} yields an arm $\bx_\mathrm{max}(\tau)$ such that, with probability at least $ 1 - \delta_\mathrm{r} $,
    \begin{align}\label{eq:bai}
        (\gs\bxs - \bx_\mathrm{max}(\tau))^\top \bthetas \leq \left(1 -\frac{\ugs}{\gs}\right) \bxs^\top \bthetas + \varepsilon \enspace,
    \end{align}
    where $\ugs$ and $\gs$ are respectively the pessimistic estimate and the maximum safe value of the safety coefficient associated to the optimal arm.
    In addition, the arms pulled during the procedure meet the overall safety requirement, that is
    \begin{align}\label{eq:safe}
        \bbP \left\{ \bigcup_{1 \leq t \leq \tau} \bx_\mathrm{pul}(t)^\top \bmus \leq \eta_0 \right\} \geq 1 - \delta_\mathrm{s} \enspace.
    \end{align}
\end{restatable}
We refer the reader to the Appendix for a detailed proof of this result.

Due to the safety threshold being approximated through a conservative estimate, the safe task cannot meet the requirements of a usual $(\varepsilon, \delta_\mathrm{r})$-BAI criterion, such as~\eqref{crit:bai}. Instead, deriving estimates from LCBs yields an estimation gap as a multiplicative error of the reward. 
This result also evidences the importance of balancing the forced-exploration phase and the constrained best-arm identification phase, as a longer forced exploration allows for a smaller $\ugs / \gs$ ratio. In practice however, performing dynamic update of the safety coefficients circumvents that issue.
Notice that such gap could be simply be transposed as an additive error through upper bounds on $\|\bthetas\|$ and $\| \bxs \|$. Using matrix concentration inequalities, we propose a refined study of the forced exploration phase in the following theorem.

\begin{restatable}{theorem}{restatefull}
    \label{thm:full-bound}
    Let $0 < \delta_\mathrm{s} < 1$ be the required confidence level of safety and $0 < \delta_\mathrm{r} < 1$ be the confidence level of best-arm identification. Let $\eiglcb(\delta_\mathrm{r})$ be defined as
    \[
        \eiglcb(\delta_\mathrm{r}) \triangleq \eigmin(\designFE) - \sqrt{\frac{2L^2}{\ug^2} \eigmin(\designFE) \log\left(\frac{2d}{\delta_\mathrm{r}}\right)} \enspace,
    \]
    where $\designFE \triangleq (\ug^2 / K) \sum_{k=1}^K \bx_k \bx_k^\top$ is the expected design matrix increment during one step of the forced exploration phase.
    
    If $\TFE$ is large enough to satisfy
    \begin{equation}
        \label{eq:tfe-bound}
        \frac{\sqrt{\TFE}}{2\betFE} \geq \frac{\left(\frac{2 \ugs \bxs^\top \bthetas}{\varepsilon} - 1\right)}{\eiglcb(\delta_\mathrm{r}) \frac{\bxs^\top \bmus}{\|\bxs\|}} \enspace,
    \end{equation}
    then running Algorithm~\ref{alg:safe-bai} with a BAI error level $\delta_\mathrm{r} / 2$ yields a solution satisfying safety~\eqref{eq:safe} requirements and such that
    \[
      \bbP\left( (\gs \bx^\star - \bx_\mathrm{max}(\tau))^\top \bthetas \leq \varepsilon \right) \geq 1 - \delta_\mathrm{r} \enspace.
    \]
\end{restatable}
The proof of the theorem is provided alongside an explicit bound on $\TFE$ in the appendix A.

The amount of time steps needed in the forced exploration phased is quantified by the RHS in~\eqref{eq:tfe-bound}. As expected, the task is harder as the required precision increases but, most importantly, the complexity of the task is also driven by $\bxs^\top \bmus / \| \bxs \|$, that is how unsafe the optimal direction really is. Finally, let us note that even though the term $\eiglcb(\deltr)$ highlights the depency between the exploration time and how well the arms span the space, this depency could easily be lifted by using refined exploration strategies such as E-optimal design.

\section{Experimental Illustrations}\label{sec:experiments}

\subsection{Experimental setup: The modified usual hard instance}

Our experiments are conducted on an adapted version of the well-known hard BAI instance introduced by~\citet{soare2014linear}. That instance allows to assess the ability of a BAI algorithm in leveraging the linear structure of the rewards to solve the problem. Precisely, we consider a set $\mathcal{X}$ of $d+1$ arms where the $d$ first  arms correspond to the canonical basis. The last arm is a slightly perturbed version of the first arm: $a_{d+1}=(\cos \omega, \sin \omega, 0, \ldots, 0)^\top$, where $\omega$ is small, in our case, we consider $\omega=0.1$. The regression parameter is given by $\btheta^\star=(2,0,\ldots,0)^\top$ (unknown to the learner). To tackle this problem instance, an efficient algorithm tends to pull the arm $a_2$ in order to reduce the uncertainty in the direction $a_1 - a_{d+1}$. To make the problem harder and add safety considerations, we limit a full access to this arm by considering a safety parameter $\mathbf{\mu} = (0,-1,0,\ldots,0)^\top$.

\subsection{Impact of different parameters} 

In all the experiments, results are averaged over 10 runs. To analyze the effect of each parameter, we only vary one parameter in each experiment. Default values are the following: dimension of the problem is $3$, hence we consider $4$ arms, parameters of the $(\varepsilon, \delta_\mathrm{r})$-best arm identification are $\varepsilon=2 (1- \cos \omega)$ and $\delta_\mathrm{r}=0.001$, noises $\epsilon_t$ and $\epsilon_t^\text{safe}$ are sampled according to a centered Gaussian distribution of standard deviation $\sigma=1.0$ and $\sigma^\text{safe}=0.1$ respectively, regularization parameter $\lambda$ is $1$, confidence parameter of the Lower Confidence Bound for safety is $\delta'_\mathrm{s}=0.001$, safety threshold $\eta_0$ is -0.5, arms are uniformly sampled during $\TFE=4 \times 200$ steps and finally, safety during this forced exploration phase is ensured by setting the parameter $\gamma$ of each arm as $0.2$. We investigate the impact of different parameters. First, our approach is compared to the standard \LGapE algorithm, Table \ref{tab:simple_comparison}. Then, in Figure \ref{fig:parameters_effect}(\textbf{a} and \textbf{b}), we displays the behaviour of the safety parameter estimation when we vary the length of initial exploration or the safety threshold. Finally, in Figure \ref{fig:parameters_effect}(\textbf{c} and \textbf{d}), we compare the impact of the dimension on the \LGapE and the \SafeLGapE.

\begin{table*}[t]
    \centering
    \caption{Impact of different parameters.}
    \begin{tabular}{@{}lccc@{}}\toprule
        Algorithm & Forced exploration & Stopping time & Percentage of unsafe arms pulled \\   
        \midrule
        \LGapE & $4 \times 200$ & $1914.7$ & $93.5\%$ \\ 
        \SafeLGapE & $4 \times 200$ & $4254.2$ & $0.0\%$ \\
        \bottomrule
    \end{tabular}
    \label{tab:simple_comparison}
\end{table*}

\begin{figure*}[t]
    \centering
    \begin{tabular}{cccc}
        a & b & c & d \\
        \includegraphics[scale=.36]{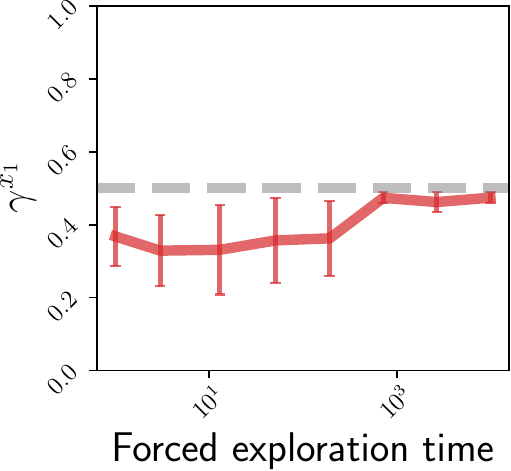} &
        \includegraphics[scale=.36]{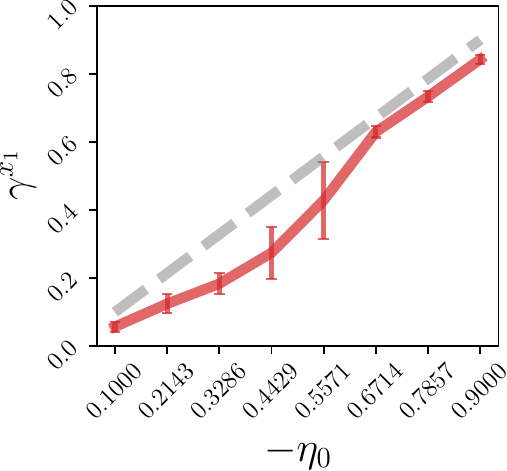} &
        \includegraphics[scale=.36]{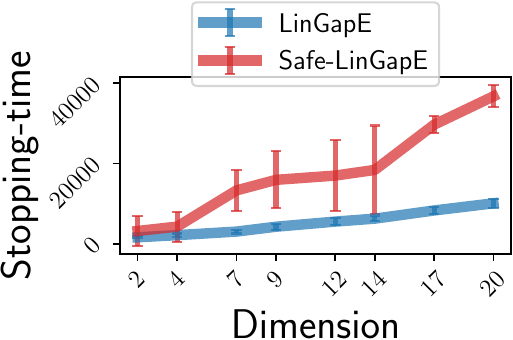} &
        \includegraphics[scale=.36]{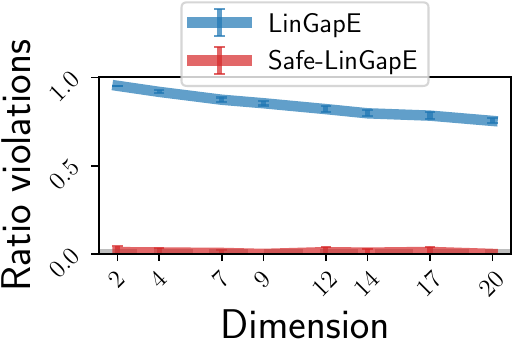} \\
    \end{tabular}
    \caption{\textbf{a} Averaged evolution of the safety parameter $\gamma_2$ of arm $a_2$ w.r.t to the length of the forced exploration time. \textbf{b} Averaged evolution of the estimated safety parameter of arm $a_1$ w.r.t to the safety threshold $\eta_0$. \textbf{c} (resp. \textbf{d}) Stopping time (resp. Ratio of safety violation) comparison between the \SafeLGapE and the \LGapE.}
    \label{fig:parameters_effect}
\end{figure*}

Table~\ref{tab:simple_comparison} confirms the intuition that \LGapE, not taking safety into account, finds the best arm quicker that \SafeLGapE. In Figure~\ref{fig:parameters_effect}(\textbf{a}), we can conclude that the longer we pursue the initial exploration the better the safety threshold is estimated and Figure~\ref{fig:parameters_effect}(\textbf{b}) confirms this intuition for various values of $\eta_0$. Figure \ref{fig:parameters_effect}(\textbf{c}) shows that increasing the dimension of the problem has a similar effect on \LGapE
and \SafeLGapE, both stopping times increasing approximately linearly on that problem. In Figure \ref{fig:parameters_effect}(\textbf{d}), we observe the stability of the safety guarantee of \SafeLGapE in increasing dimension scenarios.

\section{Conclusion}\label{sec:conclusion}

We presented a novel framework for tackling the best-arm identification task under safety constraints, by extending the regret-minimization setting. We then adapted \texttt{LinGapE} to a constrained problem, using a forced-exploration phase to gather prior information on acceptable safety levels. We provided theoretical guarantees as well as an experimental review of the performances of our method and its variants. Promising extensions include the analysis of dynamic safety coefficients and alternative criteria for selecting the ambiguous direction. 

\clearpage
\bibliography{library, refs}
\newpage
\onecolumn
\appendix

\section{LinGapE Algorithm}

$K, d > 0$, $\mathcal{X} \triangleq \{ \mathbf{x}_1, \ldots, \mathbf{x}_K \} \subseteq \bbR^d$. For any $\mathbf{x} \in \bbR^d$, $\| \mathbf{x} \|_1 \triangleq \sum_{i = 1}^d |[\mathbf{x}]_i|$.

$R$-subgaussian noise. $\| \theta \| \leq D$.

\subsection{Algorithm overview}

The method developed in~\cite{xu2018linear} operates as follows. After exploring a little bit, the following steps are repeated at each iteration:
\begin{itemize}
    \myitem{1.} the best arm (according to the Ridge estimate) is identified; \label{linegape:best-arm}
    \myitem{2.} the arm with the smallest (optimistic) gap with the best arm is identified; \label{linegape:best-gap}
    \myitem{3.} the arm to be pulled is selected according to criterion~\eqref{crit:greedy} or \eqref{crit:prob}.
\end{itemize}
The two criteria are the following. First, a greedy one
\begin{equation}
\label{crit:greedy}
    \mathbf{x}_{t + 1} \in \argmin_{\mathbf{x} \in \mathcal{X}} (\mathbf{x}_i - \mathbf{x}_j)^\top \left(\mathbf{A}_t + \mathbf{x}\mathbf{x}^\top\right)^{-1} (\mathbf{x}_i - \mathbf{x}_j) \enspace,
\end{equation}
where $\mathbf{A}_t$ is defined as in~\eqref{eq:update_mean}, and $\mathbf{x}_i$ and $\mathbf{x}_j$ are the arms identified at step~\ref{linegape:best-arm} and \ref{linegape:best-gap}, respectively. The second one, which is the only one coming with a theoretical guarantee, is
\begin{equation}
\label{crit:prob}
    \mathbf{x}_{t + 1} \in \argmin_{k \in [K]} \frac{N_t(\mathbf{x}_k)}{\mu^\star_k} \enspace,
\end{equation}
where $N_t(\mathbf{x}_k)$ is the number of times arm $\mathbf{x}_k$ has been pulled before $t$ and $\mu^\star \in \Delta_K$ is the solution of
\[
    \min_{\mu \in \Delta_K} (\mathbf{x}_i - \mathbf{x}_j)^\top \left( \mathbf{A}_t(\lambda) + \sum_{k = 1}^K \mu_k \mathbf{x}_k \mathbf{x}_k^\top \right)^{-1} (\mathbf{x}_i - \mathbf{x}_j) \enspace,
\]
and $\Delta_K \triangleq \{ \mu \in \bbR^K_+ \text{, } \sum_{k = 1}^K \mu_k = 1\}$ is the simplex in $\bbR^K$.

\subsection{Sketch of proof}

\begin{remark}[Typo in proof of Lemma 2, Appendix B]
One has $T_k(t) \geq p_k^*(y(i, j)) T_{i, j}(t)$ rather than the other way around.
\end{remark}
\begin{remark}[Typo in equation (14)]
If $i = a^*$, then the definition of $\Delta_i$ should be 
\[
    \Delta_i \triangleq \min_{j \in [K]} (x_{a^*} - x_j)^\top \theta \enspace,
\]
rather than $\argmin$.
\end{remark}

The keys elements of the proof of Theorem 2 of~\cite{xu2018linear} are contained within two lemmas. We just need to introduce a few quantities beforehand. Let $\mathbf{w}^\star(i, j) \in \bbR^K$ be the solution of
\begin{mini}
  {\mathbf{w} \in \bbR^K}{\| \mathbf{w} \|_1}{}{}
  \label{opt:weights}
  \addConstraint{\sum_{k = 1}^K [\mathbf{w}]_k \mathbf{x}_k}{=\mathbf{x}_i - \mathbf{x}_j.}{}
\end{mini}
Let $\mathbf{p}^\star(i, j) \in \bbR^K$ be defined for any $1 \leq k \leq K$ by
\[
    [\mathbf{p}^\star(i, j)]_k = \frac{| [\mathbf{w}^\star(i, j)]_k |}{\| \mathbf{w} \|_1} \enspace.
\]

The following lemma helps bounding the uncertainty associated to arms $i$ and $j$ when following criterion~\eqref{crit:prob}.
\begin{lemma}
\label{lma:uncertainty}
    For $t > 0$, when following criterion~\eqref{crit:prob}, one has for any $i, j \in [K]$:
    \[
        \| \mathbf{x}_i - \mathbf{x}_j \|_{\mathbf{A}_t^{-1}} \leq \sqrt{\frac{\| \mathbf{w}^\star(i, j) \|_1}{T_{i, j}(t)}} \enspace,
    \]
    where
    \[
        T_{i, j}(t) \triangleq \min_{k: \; [\mathbf{p}^\star(i, j)]_k > 0} \frac{N_t(\mathbf{x}_k)}{[\mathbf{p}^\star(i, j)]_k} \enspace.
    \]
\end{lemma}

\begin{remark}
    Adapting this to a safe setting should be ok. The optimization problem~\eqref{opt:weights} would be granted an additional constraint, restricting the non-zero weights to the safe set. Eventually, the optimization task would change at each iteration, but the proof of Lemma~\ref{lma:uncertainty} does not depend on the previous computations of $\mathbf{w}^\star$, only the current one. Therefore, the result should become something like that:
    \[
        \| \mathbf{x}_i - \mathbf{x}_j \|_{\mathbf{A}_t^{-1}} \leq \sqrt{\frac{\| \mathbf{w}_t^\star(i, j) \|_1}{T_{i, j}(t)}} \enspace,
    \]
    with $\mathbf{p}^\star_t(i, j)$ also depending on $t$ in the definition of $T_{i, j}(t)$. The biggest change to this lemma would actually be that $\| \mathbf{w} \|_1$ could be drastically different based on the number of arms that are considered to be safe.
\end{remark}

The second key result applies when event $\mathcal{E}$ holds, that is
\[
    \mathcal{E} \triangleq \left\{ | (\hat{\theta}_t - \theta^\star)^\top (\mathbf{x}_i - \mathbf{x}_j) | \leq \beta_{t, \delta}(i, j) \text{, for } t \geq 0 \text{ and } 1 \leq i, j \leq K \right\} \enspace,
\]
where for $t \geq 0$ and $1 \leq i, j \leq K$, $\beta_t(i, j)$ represents the uncertainty associated to the pair $(\mathbf{x}_i, \mathbf{x}_j)$. More specifically,
\[
    \beta_{t, \delta}(i, j) \triangleq \| \mathbf{x}_i - \mathbf{x}_j \|_{\mathbf{A}_t^{-1}} C_t(\delta) \enspace,
\]
and
\[
    C_t(\delta) \triangleq R \sqrt{2 \log \frac{\det(\mathbf{A}_t(\lambda))^{1/2} \det(\lambda \mathbf{I}_d)^{-1/2}}{\delta}} + D \sqrt{\lambda} \enspace.
\]
We will drop the $\delta$ index when clear from context.
\begin{lemma}
    Let us assume event $\mathcal{E}$ holds. Let $i_t$ and $j_t$ be the arms indexes identified at time $t$, during step~\ref{linegape:best-arm} and \ref{linegape:best-gap}, respectively. If either $\mathbf{x}_{i_t}$ or $\mathbf{x}_{j_t}$ is the best arm, then
    \[
        \max_{1 \leq k \leq K} \hat{\theta}_t^\top(\mathbf{x}_k - \mathbf{x}_{i_t}) + \beta_t(k, i_t) \leq \left[ \beta_t(i_t, j_t) - \max\{\Delta_{i_t}, \Delta_{j_t}\} \right]_+ + \beta_t(i_t, j_t) \enspace,
    \]
    where for $1 \leq k \leq K$,
    \[
        \Delta_k \triangleq 
        \begin{cases}
        \displaystyle\max_{k' \in [K]} (\mathbf{x}_{k'} - \mathbf{x}_k)^\top \theta^\star & \text{if $\mathbf{x}_k$ is not the best arm}\\
        \displaystyle\min_{k' \in [K]} (\mathbf{x}_k - \mathbf{x}_{k'})^\top \theta^\star & \text{otherwise} \enspace.
        \end{cases}
    \]
\end{lemma}

\section{Proofs}
In this section we prove the main result of the paper which is restated below.

\restatemain*

\subsection{Proof of Theorem~\ref{thm:bound}}
\label{app:reward-upper-bound}

\begin{proof}[\textbf{Proof of}~\eqref{eq:bai}]
We start by proving the first result in Theorem~\ref{thm:bound}.

Let $(\gamma_\star, \bx_\star) \in [0, 1] \times \mathcal{X}$ be the optimal arm, that is:
\[
    (\gamma_\star, \bx_\star) = \argmax_{(\gamma, \bx) \in [0, 1] \times \mathcal{X}}
    \left\{ \gamma \bx^\top \btheta^\star \text{, such that } \gamma \bx^\top \bmu^\star \geq \eta \right\} \enspace.
\]
Let $\tau > 0$ be the time at which the BAI procedure ends. One can then bound the suggested arm reward as follows:
\begin{align*}
    &\big(\gamma_\star \bx_\star - \bx_\mathrm{max}(\tau) \big)^\top \btheta^\star
    \leq \big(\gamma_\star - \overline{\gamma}_\star \big)\bx_\star^\top \btheta^\star
    + \big(\hat{\gamma}_\star \bx_\star - \bx_\mathrm{max}(\tau) \big)^\top \btheta^\star \\
    &\leq \big(\gamma_\star - \overline{\gamma}_\star \big)\bx_\star^\top \btheta^\star
    + \big(\hat{\gamma}_\star \bx_\star - \bx_\mathrm{max}(\tau) \big)^\top \hat{\btheta}_\tau + \| \hat{\gamma}_\star \bx_\star - \bx_\mathrm{max}(\tau) \|_{\bA_\tau^{-1}} C_\tau \\
    &\leq \big(\gamma_\star - \overline{\gamma}_\star \big)\bx_\star^\top \btheta^\star
    + \big(\bx_\mathrm{opt}(\tau) - \bx_\mathrm{max}(\tau) \big)^\top \hat{\btheta}_\tau + \| \bx_\mathrm{opt}(\tau) - \bx_\mathrm{max}(\tau) \|_{\bA_\tau^{-1}} C_\tau \\
    &= \big(\gamma_\star - \overline{\gamma}_\star \big)\bx_\star^\top \btheta^\star
    + \by_\tau^\top \hat{\btheta}_\tau + \| \by_\tau \|_{\bA_\tau^{-1}} C_\tau \,,
\end{align*}
where the second inequality holds with probability at least $1 - \delta_\mathrm{r}$.
If the stopping condition is, for $\varepsilon > 0$, $B_{\overline\bGamma}(\tau) \leq \varepsilon$ where for $t > 0$,
\[
    B_{\overline\bGamma}(t) = \by_t^\top \hat{\btheta}_t + \| \by_t \|_{\bA_t^{-1}} C_t \,,
\]
then at time $\tau$, one has
\[
    \bbP\Big(\big(\gamma_\star \bx_\star - \bx_\mathrm{max}(\tau) \big)^\top \btheta^\star \geq (1 - \alpha) r^\star + \varepsilon \Big) \leq \delta_\mathrm{r} \,,
\]
where $r^\star \triangleq \gamma_\star \bx_\star$ and $\alpha \triangleq \overline{\gamma_\star} / \gamma_\star$, which conclude the proof of the first part.
\end{proof}

\begin{proof}[\textbf{Proof of}~\eqref{eq:safe}]
The second part of the proof requires an adaptation of~\citet[Theorem 2]{xu2018linear} to our safe setting. Let us assume that it does still hold in our setting in the first place and we proceed to prove~\eqref{eq:safe}.

For the sake of simplicity, we assume that $\overline{\gamma}_1 \bx_1 = \max_{1 \leq k \leq |\mathcal{X}|} \overline{\gamma}_k \bx_k^\top \btheta^\star.$ Let $H_\mathrm{FE}(\varepsilon)$ be the problem complexity estimated after the forced-exploration phase, that is
\[
    H_\mathrm{FE}(\varepsilon) \triangleq \sum_{k = 1}^{|\mathcal{X}|} \max_{1 \leq i, j \leq |\mathcal{X}|} \frac{[\bw^\star(\bX \overline{\bGamma}, \overline{\gamma}_i \bx_i - \overline{\gamma}_j \bx_j)]_k \| \bw^\star(\bX \overline{\bGamma}, \overline{\gamma}_i \bx_i - \overline{\gamma}_j \bx_j) \|_1}{\max\left(\varepsilon, \frac{\varepsilon + \Delta_i}{3}, \frac{\varepsilon + \Delta_j}{3} \right)} \,,
\]
where for any $i \in [K]$:
\[
    \Delta_i \triangleq
    \begin{cases}
    \min_{j \neq i} (\overline{\gamma}_i \bx_i - \overline{\gamma}_j \bx_j)^\top\btheta^\star & \text{if }i = 1 \\
    (\overline{\gamma}_1 \bx_1 - \overline{\gamma}_i \bx_i)^\top \btheta^\star & \text{otherwise\,.}
    \end{cases}
\]

Combining \citet[Theorem 2]{xu2018linear} with the above definition yields:
\[
    \bbP\left(\tau \leq |\mathcal{X}| + C(\delta_\mathrm{r})\right) \geq 1 - \delta_\mathrm{r} \,,
\]
where
\[
    C(\delta_\mathrm{r}) \triangleq 8 H_\mathrm{FE}(\varepsilon) R^2 \log \frac{K^2}{\delta_\mathrm{r}} + 4 H_\mathrm{FE}(\varepsilon) R^2 d \log \left( 1 + \frac{ML^2}{\lambda d} \right) \,,
\]
and
\[
    M \triangleq 64 H_\mathrm{FE}(\varepsilon) R^4 \frac{dL^2}{\lambda} + 4\left( 8 H_\mathrm{FE}(\varepsilon) R^2 \log \frac{|\mathcal{X}|^2}{\delta} + |\mathcal{X}| \right)^2 \,.
\]
Now one can use the condition on $\delta'_\mathrm{s}$ and a union bound to conclude the proof.
\end{proof}

It remains to prove that Theorem 2 of~\citet{xu2018linear} holds in our safe setting. The key element to achieve that is to adapt Lemma 1 of~\citet{xu2018linear} to our setting. This adaptation is detailed in Lemma~\ref{lemma:bound-y} of Appendix~\ref{app:lemma}. Appendix~\ref{app:equiv} provides an essential intermediate step to that end.

\subsection{Equivalence of an optimization problem}
\label{app:equiv}
Assume that for a given couple of arms $i$ and $j$, $\mu^\star_k(i,j)$ represents the appearance ratio of arm $k$ in the sequence $\bX_n^\star(i,j)$ when $n\rightarrow\infty$, where $\bX_n^\star(i,j)$ is defined as
\[
    \bX_n^\star(i,j) \triangleq \argmin_{\bX_n} \|\bx_i-\bx_j\|_{(\bA_{\bX_n}^\lambda)^{-1}}\,.
\]
It is proved by~\cite{xu2018linear} that
\[
    [\bmu^\star(i, j)]_k = \frac{| [\mathbf{w}^\star(i, j)]_k |}{\| \mathbf{w} \|_1}\,,
\]
for which $\mathbf{w}^\star(i, j) \in \bbR^K$ is the solution of the following optimization problem:
\begin{mini*}
  {\mathbf{w} \in \bbR^K}{\| \mathbf{w} \|_1}{}{}
  \addConstraint{\sum_{k = 1}^K [\mathbf{w}]_k \mathbf{x}_k}{=\mathbf{x}_i - \mathbf{x}_j\,.}{}
\end{mini*}

The idea of the proof is the following. Let $\cY\triangleq\{\bx-\bx'|\bx,\bx'\in\cX\}$ be the set of directions. For a given $\by\in\cY$, the problem of minimizing $\|\by\|_{(\bA_{\bX_n}^\lambda)^{-1}}$ when $n\rightarrow\infty$ can be turned into the following optimization problem:

over which we can conduct a continuous relaxation:
\begin{mini}
  {\mu_k > 0,k=1,\cdots,K}{\by^\top\left(\frac{\lambda}{n}\bI_d+\sum_{k=1}^K \mu_k\bx_i\bx_i^\top\right)^{-1}\by}{}{}
  \label{opt:primal}
  \addConstraint{\sum_{k = 1}^K \mu_k=1\,.}{}
\end{mini}

The above relaxed problem is convex, thus can be solved accordingly. However, the solution would be depending on the sample size $n$. The key argument presented by~\cite{xu2018linear} to circumvent this issue is to reduce~\eqref{opt:primal} to the following problem~\eqref{opt:dual}:
\begin{mini}
  {\mu_k > 0,w_k\in\bbR,k=1,\cdots,K}{\left\|\by-\sum_{k=1}^K w_k\bx_k\right\|^2+\frac{\lambda}{n}\sum_{i=1}^K\frac{w_k^2}{\mu_k}}{}{}
  \label{opt:dual}
  \addConstraint{\sum_{k = 1}^K \mu_k=1\,.}{}
\end{mini}

The reasoning is somehow fuzzy in the paper\footnote{They refer to a result by~\cite{yu2006active}, which is not trivially applicable in our case.}. We hereby provide a detailed proof of the equivalence. Formally, we prove the following proposition.

\begin{proposition}
The continuous relaxed problem~\eqref{opt:primal} is equivalent to the optimization problem~\eqref{opt:dual}.
\end{proposition}

\begin{proof}
We first rewrite all the values in the form of matrices. First, we let
\[
  \bM =
  \begin{bmatrix}
    \mu_1 & & \\
    & \ddots & \\
    & & \mu_K
  \end{bmatrix}
  \in \bbR^{K\times K}
\]
be the diagonal matrix composed of values of $(\mu_k)_{i=1..K}$. We also let $\bmu = (\mu_1,\cdots,\mu_K)^\top$, $\bw = (w_1,\cdots,w_K)^\top$ and
\[
  \bX =
  \begin{bmatrix}
    \vdots & & \vdots \\
    \bx_1 & \cdots & \bx_K \\
    \vdots & & \vdots
  \end{bmatrix}
  \in \bbR^{d\times K}
\]
be the matrix that concatenates all the column vectors $(\bx_i)_{i=1..K}$.

Now, the problem~\eqref{opt:dual} can be rewritten as
\begin{mini}
  {\bmu\in\left]0,1\right[^K,\bw\in\bbR^K}{\left\|\by-\bX\bw\right\|^2+\frac{\lambda}{n}\bw^\top\bM^{-1}\bw}{}{}
  \label{opt:dual_prime}
  \addConstraint{\sum_{k = 1}^K \mu_k=1\,.}{}
\end{mini}
By taking the partial derivatives of the cost function of~\eqref{opt:dual_prime} with respect to $\bw$, the minimum of~\eqref{opt:dual_prime} is reached when
\[
    2\bX^\top\bX\bw - 2\bX^\top\by + 2\frac{\lambda}{n}\bM^{-1}\bw = 0\,.
\]
Thus we have
\begin{align}\label{eq:min}
    \bw^\star &= \left(\frac{\lambda}{n}\bM^{-1}+\bX^\top\bX\right)^{-1}\bX^\top\by\nonumber\\
    &= \bM^{\frac{1}{2}}\bM^{-\frac{1}{2}}\left(\frac{\lambda}{n}\bM^{-1}+\bX^\top\bX\right)^{-1}\bM^{-\frac{1}{2}}\bM^{\frac{1}{2}}\bX^\top\by\nonumber\\
    &= \bM^{\frac{1}{2}}\left(\frac{\lambda}{n}\bI_K+\bM^{\frac{1}{2}}\bX^\top\bX\bM^{\frac{1}{2}}\right)^{-1}\bM^{\frac{1}{2}}\bX^\top\by\,.
\end{align}
$\bM$ is a positive definite matrix, so is $\bM^{1/2}$, and in particular, $\bM^{1/2}$ is symmetric. Let $\bQ = \bX\bM^{1/2}$ and apply the \emph{matrix inversion lemma} (or Woodbury formula), we obtain
\begingroup
\allowdisplaybreaks
\begin{align*}
    \eqref{eq:min} &= \bM^{\frac{1}{2}}\left(\frac{\lambda}{n}\bI_K+\bQ^\top\bQ\right)^{-1}\bQ^\top\by\nonumber\\
    &= \bM^{\frac{1}{2}}\bQ^\top\left(\frac{\lambda}{n}\bI_d+\bQ\bQ^\top\right)^{-1}\by\nonumber\\
    &= \bM\bX^\top\left(\frac{\lambda}{n}\bI_d+\bX\bM\bX^\top\right)^{-1}\by\,.
\end{align*}
\endgroup

Finally, it remains to plug $\bw^\star = \bM\bX^\top\left((\lambda/n)\bI_d+\bX\bM\bX^\top\right)^{-1}\by$ into~\eqref{opt:dual}, and we have
\begingroup
\allowdisplaybreaks
\begin{align*}
    &\left\|\by-\bX\bw^\star\right\|^2+\frac{\lambda}{n}(\bw^\star)^\top\bM^{-1}\bw^\star\\
    = &\ (\by-\bX\bw^\star)^\top(\by-\bX\bw^\star)+\frac{\lambda}{n}(\bw^\star)^\top\bM^{-1}\bw^\star\\
    = &\ \by^\top\by - \by^\top\bX\bw^\star - (\bw^\star)^\top\bX^\top\by + (\bw^\star)^\top\bX^\top\bX\bw^\star + \frac{\lambda}{n}(\bw^\star)^\top\bM^{-1}\bw^\star\\
    = &\ \by^\top\by - \by^\top\bX\bM\bX^\top\left(\frac{\lambda}{n}\bI_d+\bX\bM\bX^\top\right)^{-1}\by - \by^\top\left(\frac{\lambda}{n}\bI_d+\bX\bM\bX^\top\right)^{-1}\bX\bM\bX^\top\by\\
    &\ + \by^\top\left(\frac{\lambda}{n}\bI_d+\bX\bM\bX^\top\right)^{-1}\bX\bM\bX^\top\bX\bM\bX^\top\left(\frac{\lambda}{n}\bI_d+\bX\bM\bX^\top\right)^{-1}\by\\
    &\ + \frac{\lambda}{n}\by^\top\left(\frac{\lambda}{n}\bI_d+\bX\bM\bX^\top\right)^{-1}\bX\bM\bM^{-1}\bM\bX^\top\left(\frac{\lambda}{n}\bI_d+\bX\bM\bX^\top\right)^{-1}\by\\
    = &\ \by^\top\by - \by^\top\bX\bM\bX^\top\left(\frac{\lambda}{n}\bI_d+\bX\bM\bX^\top\right)^{-1}\by - \by^\top\left(\frac{\lambda}{n}\bI_d+\bX\bM\bX^\top\right)^{-1}\bX\bM\bX^\top\by\\
    &\ + \by^\top\left(\frac{\lambda}{n}\bI_d+\bX\bM\bX^\top\right)^{-1}\left(\bX\bM\bX^\top+\frac{\lambda}{n}\bI_d\right)\bX\bM\bX^\top\left(\frac{\lambda}{n}\bI_d+\bX\bM\bX^\top\right)^{-1}\by\\
    = &\ \by^\top\by - \by^\top\bX\bM\bX^\top\left(\frac{\lambda}{n}\bI_d+\bX\bM\bX^\top\right)^{-1}\by - \by^\top\left(\frac{\lambda}{n}\bI_d+\bX\bM\bX^\top\right)^{-1}\bX\bM\bX^\top\by\\
    &\ + \by^\top\bX\bM\bX^\top\left(\frac{\lambda}{n}\bI_d+\bX\bM\bX^\top\right)^{-1}\by\\
    = &\ \by^\top\by - \by^\top\left(\frac{\lambda}{n}\bI_d+\bX\bM\bX^\top\right)^{-1}\bX\bM\bX^\top\by\\
    = &\ \by^\top\left(\frac{\lambda}{n}\bI_d+\bX\bM\bX^\top\right)^{-1}\left(\frac{\lambda}{n}\bI_d+\bX\bM\bX^\top\right)\by - \by^\top\left(\frac{\lambda}{n}\bI_d+\bX\bM\bX^\top\right)^{-1}\bX\bM\bX^\top\by\\
    = &\ \frac{\lambda}{n}\by^\top\left(\frac{\lambda}{n}\bI_d+\bX\bM\bX^\top\right)^{-1}\by\,,
\end{align*}
\endgroup
which completes the proof.
\end{proof}

\subsection{Key lemma}
\label{app:lemma}

Throughout this section, we assume that the safety thresholds has been previously estimated, leading to estimates $(\overline{\gamma}_1, \ldots, \overline{\gamma}_K)$.

At any round $t > 0$, we denote the safety values as $\overline{\bGamma} \triangleq \Diag(\overline{\gamma}_1, \ldots, \overline{\gamma}_K)$, the diagonal matrix constructed from the LCB safety estimates. In addition, we denote as $\by_t \triangleq \bx_\mathrm{max}(t) - \bx_\mathrm{opt}(t)$ the selected direction to disambiguate.

Now for any $\by \in \bbR^d$ and any $\bX \in \bbR^{d \times K}$, we are interested in the asymptotic case of~\eqref{opt:dual} from Appendix~\ref{app:equiv}, and we define the following optimization problem denoted by $\mathrm{P}(\bX, \by)$:
\begin{mini}
  {\bw, \bnu \in \bbR^K}{\sum_{k : w_k > 0} \frac{w_k^2}{\nu_k}}{}{}
  \addConstraint{\bX^\top \bw}{=\by \,}{}
  \addConstraint{\bnu^\top \1_K}{=1 \,}{}
  \addConstraint{\bnu}{\geq 0 \,.}{}
  \label{opt:simple-weights}
\end{mini}
We also denote as $\bw^\star(\bX, \by)$ and $\bnu^\star(\bX, \by)$ the associated primal solutions. At each round $t > 0$, we focus on the optimization problem~$\mathrm{P}(\bX \overline{\bGamma}, \by_t)$; for the sake of simplicity, we denote as $\bw^\star(t)$ and $\bnu^\star(t)$ the primal solutions to this problem.

For any arm $k \in [K]$, we denote as $N_k(t)$ the scaled number of times it has been pulled until round $t$, that is
\[
    N_k(t) \triangleq \sum_{s = 1}^t \mathbb{I}\{k \text{ is pulled at round } s\} \; \hat{\gamma}_k^2 \,.
\]
Notice that using this notation allows us to write the design matrix as
\[
    \bA_t = \lambda \bI + \sum_{k = 1}^K N_k(t) \bx_k \bx_k^\top \,.
\]
The following lemma adapts a key result from \citet{xu2018linear} to the safe setting.
\begin{lemma}\label{lemma:bound-y}
    \[
        \| \by_t \|_{\bA_t^{-1}} \leq \frac{\| \bw^\star(t) \|_1}{\sqrt{N_{\by_t}(t)}} \,,
    \]
    where
    \[
        N_{\by_t}(t) \triangleq \min_{k : w_k^\star(t) \neq 0} \; \frac{N_k(t)}{\hat{\gamma}_k^2|w_k^\star(t)|} \enspace.
    \]
\end{lemma}

\begin{proof}
    The proof is a straightforward adaptation of \citet[Lemma 1]{xu2018linear}. Writing the KKT conditions of problem~$\mathrm{P}(\bX \hat{\bGamma}, \by_t)$ yields
    \[
        w_k^\star(t) = \frac{1}{2} \nu_k^\star(t) \hat{\gamma}_k \; \bx_k^\top \bphi_1 =  \mathrm{sgn}(w_k^\star (t)) \, \nu_k^\star(t) \, \| \bw^\star(t) \|_1 \,,
    \]
    where $\bphi_1$ is the Lagrange multiplier of the first equality constraint. Similarly,
    \[
        \by_t = \frac{1}{2} \sum_{k = 1}^K \hat{\gamma}_k^2 \, \nu_k^\star(t) \bx_k \bx_k^\top \bphi_1 \,.
    \]
    Then, one can write
    \begin{align*}
        \by_t^\top \bA_t^{-1} \by_t &= \by_t^\top \left( \lambda \bI + \sum_{k = 1}^K N_k(t) \bx_k \bx_k^\top \right)^{-1} \by_t \\
        &\leq \by_t^\top \left( \lambda \bI + N_{\by_t}(t) \| \bw^\star (t) \|_1 \sum_{k = 1}^K \nu_k^\star(t) \hat{\gamma}_k^2 \, \bx_k \bx_k^\top \right)^{-1} \by_t \\
        &= \frac{1}{2} \bphi_1^\top \left( \sum_{k = 1}^K \nu_k^\star(t) \hat{\gamma}_k^2 \, \bx_k \bx_k^\top \right) \left( \lambda \bI + N_{\by_t}(t) \| \bw^\star (t) \|_1 \sum_{k = 1}^K \nu_k^\star(t) \hat{\gamma}_k^2 \, \bx_k \bx_k^\top \right)^{-1} \by_t \\
        &\leq \frac{\bphi_1^\top \by_t}{2 N_{\by_t}(t) \, \| \bw^\star (t) \|_1}  \\
        &= \frac{1}{4 N_{\by_t}(t) \| \bw^\star (t) \|_1 } \; \bphi_1^\top \left(\sum_{k = 1}^K \nu_k^\star(t) \hat{\gamma}_k^2 \, \bx_k \bx_k^\top \right) \bphi_1 \\
        &= \frac{1}{4 N_{\by_t}(t) \| \bw^\star (t) \|_1 } \; \bphi_1^\top \left(\sum_{k = 1}^K \nu_k^\star(t) \hat{\gamma}_k^2 \, \bx_k \bx_k^\top \right) \bphi_1 \\
        &= \frac{\| \bw^\star(t) \|_1^2}{N_{\by_t}(t) \|\bw^\star(t)\|_1} \enspace,
    \end{align*}
    and the result holds.
\end{proof}
\subsection{Bounding the safety estimates}
\label{sec:safety-estimates}

The main issue with result~\eqref{eq:bai} from Theorem~\ref{thm:bound} is the fact that it is algorithm dependent, since $\ugs$ is a pessimistic estimates of the safety coefficient $\gs$ and thus depends on the safety exploration policy. In our case, the exploration is a uniform random sampling on the directions during $\TFE$ steps. Formally, $\ugs$ is defined as follows:
\[
  \ugs \triangleq 1 - \left( 1 - \frac{\eta_0}{ \bys^\top \bmuh - \betFE \| \bys \|_{\bAFE^{-1}} } \right)_+ \enspace,
\]
where
\[
  \betFE \triangleq R \sqrt{d \log\left( \frac{1 + tL^2/\ug^2\lambda}{\delts}\right)} + D\sqrt{\lambda} \enspace.
\]

Let us first assume that $\gs < 1$, so $\ugs < 1$ with probability at least $1 - \delts$. One thus has, with high probability,
\begin{align*}
  \frac{\ugs}{\gs} &= \frac{\bys^\top \bmus }{\bys^\top \bmuh - \betFE \| \bys \|_{\bAFE^{-1}}}\\
                   &\geq \frac{\bys^\top \bmus}{\bys^\top \bmus + 2\betFE \| \bys \|_{\bAFE^{-1}}} \enspace.
\end{align*}

\paragraph{Ways of writing $\bAFE$.}
The design matrix can be written in several ways, depending on the point of view, that is
\[
  \bAFE \triangleq \sum_{t = 1}^{\TFE} \ug^2 \bx_t \bx_t^\top = \ug^2 \sum_{t = 1}^{\TFE} \by_t \by_t^\top = \ug^2 \sum_{k = 1}^K N_k(\TFE) \bx_k \bx_k^\top \enspace,
\]
where $\ug$ is the common lower bound of the safety coefficient on all directions. Let $\designFE$ be the expectation of the design matrix after one exploration step, that is
\[
  \designFE \triangleq \frac{\ug^2}{K} \sum_{k = 1}^K \bx_k \bx_k^\top \enspace.
\]
Using Hoeffding's inequality applied on random matrices (see, \eg~\citealt{rizk2020refined}), one obtains
\[
  \bbP\big( \eigmin(\bAFE) \leq (1 - \varepsilon) \TFE \, \eigmin\left(\designFE \right) \big) \leq d \exp \left( -\frac{\varepsilon^2 \eigmin(\designFE)}{2 \ug^2 L^2} \right) \enspace,
\]
where $L \triangleq \max_{\by \in \cY} \| \by \|$. Therefore, for a given level of error $\delta$ such that
\[
  d \exp\left(\frac{-\eigmin(\designFE)}{2 \ug^2 L^2}\right) < \delta < 1 \enspace,
\]
one has with probability at least $1 - \delta$:
\[
  \eigmin(\bAFE) \geq \TFE \, \eigmin(\designFE)\left(1 - 2\ug^2 L^2\log\left(\frac{d}{\delta}\right)\right) \triangleq \TFE \eiglcb(\delta) \enspace.
\]
In particular, for any $\bx \in \cX$, one has
\[
  \| \bx \|_{\bAFE^{-1}} = \sqrt{\bx^\top \bAFE^{-1} \bx} \leq \frac{\| \bx \|}{\sqrt{\eigmin(\bAFE)}} \leq \frac{\| \bx \|}{\sqrt{\TFE \eiglcb(\delta)}} \enspace.
\]
Combining this with our previous inequality yields
\begin{align*}
  1 - \frac{\ugs}{\gs} &\leq 1 - \frac{\bys^\top \bmus}{\bxs^\top \bmus + 2\betFE \| \bxs \|_{\bAFE^{-1}}} \\
                       &= 1- \frac{\bxs^\top \bmus}{\bxs^\top \bmus + 2\betFE \frac{\| \bxs \|}{\sqrt{\TFE \eiglcb(\delta)}}} \\
                       &= 1 - \frac{1}{1 + \frac{2 \betFE \| \bxs \|}{\bxs^\top \bmus \sqrt{\TFE \eiglcb(\delta)}}} \enspace.
\end{align*}
Our goal is to balance the two RHS terms in Theorem~\ref{thm:bound}. Let us fix the error threshold of forced exploration to $\varepsilon / 2$. One must then set $\TFE$ such that
\[
  1 - \frac{1}{1 + \frac{2 \betFE \| \bxs \|}{\TFE \eiglcb(\delta) \bxs^\top \bmus}} \leq \frac{\varepsilon}{2 \ugs \bxs^\top \bthetas}\enspace.
\]
If $\varepsilon \leq 2 \ugs \bxs^\top\bthetas$, this condition can be expressed as follows:
\[
  \frac{\sqrt{\TFE \eiglcb(\delta)}}{\betFE} \geq 2 \frac{\| \bxs \|}{\bxs^\top \bmus} \left( \frac{2\ugs \bxs^\top \bthetas}{\varepsilon} - 1 \right) \enspace.
\]
Using the above inequality, one can guarantee that the gap between the safety threshold LCB and its optimal value is small enough to ensure an estimation error of at most $\varepsilon / 2$, as long as
\[
  \sqrt{\TFE} \geq -2 \frac{\| \bxs \|}{\bxs^\top \bmus} \left( \frac{2\ugs \bxs^\top \bthetas}{\varepsilon} - 1 \right)
  W\left(-\delts\frac{\exp\left(-\frac{\varepsilon \bxs^\top\bmus}{2 \| \bxs \| \big(2\ugs \bxs^\top\bthetas - \varepsilon \big)}\right)}{\frac{2 \| \bxs \| \big(2\ugs \bxs^\top\bthetas - \varepsilon \big)}{\varepsilon \bxs^\top\bmus}}\right) - 1 \enspace,
\]
where $W$ is the product logarithm function.

\end{document}